\documentclass[12pt]{article}
\usepackage[latin9]{inputenc}
\usepackage{units}
\usepackage{amsmath}
\usepackage{amssymb}
\usepackage{amsthm}
\usepackage{overpic}
\usepackage{multirow}
\usepackage{rotating}

\makeatletter

\usepackage{graphicx} 

\usepackage{amsfonts}
\usepackage{algorithmic}\usepackage{algorithm}
\newtheorem{theo}{Theorem}[section]

\newtheorem{corollary}{Corollary}[section]
\newtheorem{conjecture}[theo]{Conjecture}

\newcommand{\bb}[1]{\boldsymbol{\mathrm{#1}}}

\def\Tr{\mathrm{T}}

\newcommand{\RR}{\mathbb{R}}

\newcommand{\uu}{\bb{u}}

\newcommand{\xx}{\bb{x}}

\newcommand{\Aa}{\bb{A}}
\newcommand{\Bb}{\bb{B}}

\newcommand{\Xx}{\bb{X}}
\newcommand{\Yy}{\bb{Y}}

\newcommand{\Uu}{\bb{U}}
\newcommand{\Dd}{\bb{D}}
\newcommand{\Ii}{\bb{I}}

\newcommand{\Ww}{\bb{W}}
\newcommand{\Ll}{\bb{L}}

\newcommand{\tr}{\mathrm{tr}\,}

\makeatother

\begin{document}

\title{Making Laplacians commute}

\author{Michael M. Bronstein$^1$ \and Klaus Glashoff$^1$ \and Terry A. Loring$^2$\\
\small $^1$Institute of Computational Science, Faculty of Informatics,\\
\small Universit{\`a} della Svizzera Italiana, Lugano, Switzerland \vspace{1mm}\\
\small $^2$Department of Mathematics and Statistics, \\
\small University of New Mexico, Albuquerque (NM), USA
}
\maketitle
\begin{abstract}
In this paper, we construct multimodal spectral geometry by finding a pair of closest commuting operators (CCO) to a given pair of Laplacians. 
The CCOs are jointly diagonalizable and hence have the same eigenbasis. 
%
Our construction naturally extends classical data analysis tools based on spectral geometry, such as diffusion maps and spectral clustering. 
We provide several synthetic and real examples of applications in dimensionality reduction, shape analysis, and clustering, demonstrating that our method better captures the inherent structure of multi-modal data. 

\end{abstract}

\section{Introduction}

Spectral methods proved to be an important and versatile tool in a wide range of problems in the fields of computer graphics, machine learning, pattern recognition, and computer vision. 
In computer graphics and geometry processing, classical signal processing methods based on frequency transforms were generalized to non-Euclidean spaces (Riemannian manifolds), where the eigenfunctions of the Laplace-Beltrami operator act as a non-Euclidean analogy of the Fourier basis, allowing one to perform harmonic analysis on the manifold. Applications based on such approaches include shape compression \cite{karni2000spectral}, filtering \cite{Levy09}, pose transfer \cite{Levy:2006:LET:1136647.1136965,rong2008spectral}, symmetry detection \cite{Ovsjanikov08}, shape description \cite{Sun09,gebal:andreas:etal:adf:09,BroSI:HKS,litman:BB:11,aubry2011wave}, retrieval \cite{BroBB1,BroBroOvsGui11}, and correspondence \cite{ovsjanikov2010one,dubrovina2010matching,ovsjanikov2012functional}. 

In pattern recognition, one can think of the data as a low-dimensional manifold embedded into a high-dimensional space, whose local intrinsic structure is represented by the Laplace-Beltrami operator. In the discrete version, the manifold is represented as a graph and the Laplace-Beltrami operator as a graph Laplacian.  
Many problems thus boil down to finding the first eigenfunctions of the Laplacian: for example, in spectral clustering~\cite{Ng01onspectral} clusters
are determined by the smallest eigenfunctions of the Laplacian; 
eigenmaps~\cite{Belkin02laplacianeigenmaps}
and diffusion maps \cite{Coifman05geometricdiffusions,Nadler05diffusionmaps} 
embed the manifold into a low-dimensional space using the smallest eigenfunctions 
of the Laplacian or the related heat operator; and diffusion metrics~\cite{Coifman05geometricdiffusions} measure the distances in this low-dimensional space. 
Other examples 
include 
spectral graph partitioning~\cite{DingHeetal2001}, 
spectral hashing~\cite{weiss2008spectral},   
 image segmentation~\cite{Shi97normalizedcuts}, spectral correspondence, and shape analysis. 

%

{\bf Multimodal spectral geometry. }
Many data analysis applications involve observations and measurements of data 
using different modalities, such as multimedia documents~\cite{bekkerman2005multi,weston2010large,rasiwasia2010new,mcfee2011learning}, audio and video~\cite{kidron2005pixels,alameda2011finding,sharma2012generalized}, 
images with different lighting conditions \cite{bansaljoint}, 
or medical imaging modalities \cite{bronstein2010data}.  
In shape analysis applications, it is important to be able to design compatible bases on multiple shapes, e.g. in order to transfer functions or vector fields from one shape to another \cite{KovBBKK:2013:EG}.

While spectral methods have been extensively studied for a single data space (manifold), there have been relatively few attempts of principled and systematic extension of spectral methods to multimodal settings involving multiple data spaces. 
In particular, problems of multimodal (or `multi-view') clustering have gained increasing interest in the computer vision and pattern recognition community \cite{de2005spectral,MaAnchor2008,tang2009clustering,Cai2011,kumarco,dong2013clustering}.  
%
Sindhwani et al. \cite{sindhwani2005co} used a convex combination of Laplacians in the `co-regularization' framework.  
{\em Manifold alignment} considered multiple manifolds as a single space with `connections' between points and tries to find an aligned set of eigenvectors \cite{ham2005semisupervised,wangageneral2009,wang2008manifold}. 
A similar philosophy has been followed in the recent work of Eynard et al. \cite{eynard2012multimodal}, who proposed an extension of the spectral methods to the multimodal setting by finding a common eigenbasis of multiple Laplacians by means of joint approximate diagonalization~\cite{Bunse-Gerstner:1993,CardosoBlind1993,Cardoso96jacobiangles,YeredorJD202,ZieglerBlindsource05}. 
They also showed that many previous methods for multi-modal clustering can be developed as instances of the joint diagonalization framework. 
Kovnatsky et al. \cite{KovBBKK:2013:EG} used joint diagonalization in computer graphics and shape analysis problems.

{\bf Main contribution. }
In this paper, we study a class of methods we term closest commuting operators (CCO), which we show to be equivalent to joint diagonalization. 
However, one of the main drawbacks of joint diagonalization is that when applied to Laplacians, it does not preserver their structure. On the other hand, with the CCO problem, we can restrict our search to the set of legal Laplacian matrices, thus finding {\em closest commuting Laplacians} with the same sparse structure rather than arbitrary matrices. 
We show that such optimization produces meaningful multimodal spectral geometric constructions.

The rest of the paper is organized as follows. 
In Section 2, we provide the background on spectral geometry of graphs. 
In Section 3, we formulate our CCO problem. For the simplicity of discussion, we consider undirected graphs with equal vertex sets and unnormalized Laplacians.  We discuss the relation between joint diagonalization and closest commuting matrices and show that the two problems are equivalent. 
Section 4 is dedicated to numerical implementation. 
In Section 5, we discuss the generalization to the setting of different vertex sets using the notion of functional correspondence. 
Section 6 presents experimental results. 
Finally, Section 7 concludes the paper.

\section{Background}

\subsection{Notation and definitions}

Let $\Aa, \Bb$ be two $n\times n$ real matrices. We denote by 
\begin{eqnarray*}
\| \Aa \|_\mathrm{F} &=& \textstyle \left( \sum_{ij}|a_{ij}|^2\right)^{1/2} = \left( \mathrm{tr}(\Aa^\Tr \Aa) \right)^{1/2}; \\
\| \Aa \|_2 &=& \max_{\xx \in \mathbb{R}^n : \| \xx\|_2 = 1}\| \Aa \xx \|_2 = \left( \lambda_{\mathrm{max}}(\Aa^\Tr\Aa) \right)^{1/2}, 
\end{eqnarray*} 
the {\em Frobenius} and the {\em operator} norm (induced by the Euclidean vector norm) of $\Aa$, respectively. 
We say that $\Aa$ and $\Bb$ {\em commute} if $\Aa\Bb = \Bb\Aa$, and call $[\Aa,\Bb] = \Aa\Bb - \Bb\Aa$ their {\em commutator}. 
If there exists a unitary matrix $\hat{\bb{U}}$ such that $\hat{\bb{U}}^\Tr \Aa \hat{\bb{U}} = \bb{\Lambda}_A$ and $\hat{\bb{U}}^\Tr \Bb \hat{\bb{U}} = \bb{\Lambda}_B$ are diagonal, we say that $\Aa, \Bb$ are {\em jointly diagonalizable} and call such $\hat{\bb{U}}$ the {\em joint eigenbasis} of $\Aa$ and $\Bb$. 
We denote by $\mathrm{diag}(\Aa)$ a column vector containing the diagonal elements of matrix $\Aa$, and by $\mathrm{diag}(a_1,\hdots, a_n)$ a diagonal matrix containing on the diagonal the elements $a_1, \hdots, a_n$. Furthermore, we use $\mathrm{Diag}(\Aa) = \mathrm{diag}(\mathrm{diag}(\Aa))$ to denote a diagonal matrix obtained by setting to zero the off-diagonal elements of $\Aa$.

\subsection{Spectral geometry}

Let us be given an undirected graph $G = (V,E)$ with vertices $V = \{x_1, \hdots, x_n\}$ and weighted edges $E \subseteq \{1,\hdots, n \}^2$ with weights $w_{ij} \geq 0$. We say that vertices $x_i, x_j$ are {\em connected} if $(i,j) \in E$, or alternatively, $w_{ij}>0$. 
%
%
The $n\times n$ matrix $\Ww = (w_{ij})$ is called the {\em adjacency matrix}  and 
\begin{eqnarray}
\Ll = \Dd - \Ww, \hspace{5mm} \Dd = \mathrm{diag}\left(\sum_{l = 1, l\neq i}^n w_{il} \right)
\label{eq:laplacian}
\end{eqnarray} 
the (unnormalized) {\em Laplacian} of $G$. Since in undirected graph $(i,j) \in E$ implies $(j,i) \in E$, the matrices $\bb{W}$ and $\bb{L}$ are symmetric. 
%
Consequently, $\Ll $ admits the unitary eigendecomposition $\bb{L} = \bb{\Phi} \bb{\Lambda} \bb{\Phi}^\Tr$ with orthonormal eigenvectors $\bb{\Phi} = (\bb{\phi}_1,\hdots, \bb{\phi}_n)$ and real eigenvalues $0=\lambda_1 \leq \lambda_2 \leq \hdots \leq \lambda_n$, $\bb{\Lambda} = \mathrm{diag}(\lambda_1,\hdots, \lambda_n)$.

Spectral graph theory \cite{chung1997spectral} studies the properties of the graph through analyzing the spectral properties of its Laplacian.  
%
It is closely related to spectral geometry of Riemannian manifolds \cite{ber:gau:maz:71:DRUM}, of which the graphs can be thought of as a discretization, and the Laplacian matrix corresponds to the Laplace-Beltrami operator on a Riemannian manifold. 
In particular, spectral methods have been successfully applied in the field of machine learning and shape analysis. 
We outline below the main spectral geometric constructions to which we will refer later in the paper.

{\bf Fourier analysis on graphs.} 
Given a function $f : V \rightarrow \RR$ defined on the vertices of the graph and represented as the $n$-dimensional column vector $\bb{f} = (f(x_1), \hdots, f(x_n))^\Tr$, we can decompose it in the orthonormal basis of the Laplacian eigenvectors $\bb{\phi}_1,\hdots, \bb{\phi}_n$ using {\em Fourier series}, 
$$
f(x_p) = \sum_{i=1}^n \langle \bb{f}, \bb{\phi}_i \rangle \phi_{pi}, 
$$
or in matrix notation, $\bb{f} = \bb{\Phi}\bb{\Phi}^\Tr \bb{f}$.

{\bf Heat diffusion on graphs.} Similarly to the standard heat diffusion equation, one can define a diffusion process on $G$, governed by the following PDE: 
$$
\bb{L} \bb{f}(t) + \frac{\partial}{\partial t}\bb{f}(t) = 0, \hspace{5mm} \bb{f}(0) = \bb{u}, 
$$
where the solution $\bb{f}(t): V \times [0,\infty) \rightarrow \RR_+$ 
is the amount of heat at time $t$ at the vertices $V$.  
The solution of the heat equation is given by $\bb{f}(t) = e^{-t \Ll}\bb{u}$, and one can easily verify that it satisfies the heat equation $(\Ll e^{-t \Ll} - \Ll e^{-t \Ll}) \bb{u} = 0$ and the initial condition $\bb{f}(0) = e^{-0\Ll}\bb{u} = \bb{u}$.
The matrix 
$$\bb{H}^t = e^{-t \Ll} = \bb{\Phi} e^{-t \bb{\Lambda}} \bb{\Phi}^\Tr$$  
is called the {\em heat operator} (or the {\em heat kernel}) and can be interpreted as the `impulse response' of the heat equation. 

{\bf Diffusion maps.} 
Embeddings by means of the heat kernel have been studied by B\'{e}rard et al. \cite{berard1994embedding} and  Coifman et al. \cite{Coifman05geometricdiffusions, Coifman}.  
In the context of non-linear dimensionality reduction, Belkin and Niyogi \cite{Belkin02laplacianeigenmaps,TT_JCSS_08} showed that finding
a neighborhood-preserving $m$-dimensional embedding of the graph can be posed as the
{\em minimum eigenvalue problem},
\begin{eqnarray}
\min_{\Uu \in\RR^{n\times m}}\tr(\Uu^{\Tr}\Ll \Uu) \,\,\, \mathrm{s.t.} \,\,\, \Uu^{\Tr}\Uu=\Ii, \label{eq:belkin}
\end{eqnarray}
which has an analytic solution $\Uu = (\bb{\phi}_1, \hdots, \bb{\phi}_m)$, referred to as {\em Laplacian eigenmap}. 
%
The neighborhood-preserving property of the eigenmaps is related to the fact the the smallest `low-frequency' eigenvectors of the Laplacian vary smoothly on the vertices of the graph.

More generally, a {\em diffusion map} is given as a mapping of the form 
$\Uu = \left(K(\lambda_1) \bb{\phi}_{1}, \hdots, K(\lambda_m) \bb{\phi}_{m} \right)$,
where $K(\lambda)$ is some transfer function acting as a `low-pass filter' on eigenvalues $\lambda$ \cite{Coifman05geometricdiffusions, Coifman}. 
In particular, the setting $K(\lambda) = e^{-t\lambda}$ corresponds to heat kernel embedding. 



{\bf Diffusion distances.} 
Coifman et al. \cite{Coifman05geometricdiffusions, Coifman} defined the {\em diffusion distance} as a `cross-talk' between the heat kernels 
\begin{eqnarray}
d_t(x_p,x_q) = \left( \sum_{i=1}^n ((\bb{H}^t)_{pi} - (\bb{H}^t)_{qi})^2  \right)^{1/2} = \left( \sum_{i=1}^n e^{-2t \lambda_i } \phi_{pi} \phi_{qi} \right)^{1/2}.    
\label{eq:diffdist}
\end{eqnarray}
Intuitively, $d_t(x_p,x_q)$ measures the `reachabilty' of vertex $x_p$ from $x_q$ by a heat diffusion of length $t$.

{\bf Spectral clustering. }
Ng et al. \cite{Ng01onspectral} showed a very efficient and robust clustering approach based on the observation that the multiplicity of the null  eigenvalue of $\Ll$ is equal to the number of connected components of $G$. The corresponding eigenvectors act as indicator functions of these components.
  Embedding the data using the null eigenvectors of $\Ll$ and then applying some standard clustering algorithm such as K-means was shown to produce significantly better results than clustering the high-dimensional data directly.
  %


\subsection{Joint diagonalization }

In many data analysis applications, we have multiple modalities or `views' of the same data, which can be considered as graphs with different connectivities (sometimes referred to as {\em multi-layered graphs} \cite{dong2013clustering}) $G_k = (V,E_k),\, k = 1, 2$ with equal set of $|V| = n$ vertices and different weighted edges, with corresponding adjacency matrices $\Ww_k = (w^k_{ij} \geq 0)$  and Laplacians $\Ll_k= \Dd_k - \Ww_k$. 
We denote their respective eigenvalues by $\bb{\Lambda}_k = \mathrm{diag}(\lambda_{k,1}, \hdots, \lambda_{k,n})$ and eigenvectors by $\bb{\Phi}_k = (\bb{\phi}_1^k, \hdots, \bb{\phi}_n^k)$, and the heat operators by $\bb{H}^t_k = \bb{\Phi}_k e^{-t\bb{\Lambda}} \bb{\Phi}_k^\Tr$.

The main question treated in this paper is how to generalize the spectral geometric constructions to such a setting, obtaining a single object such as diffusion map or distance from multiple graphs. 
Eynard et al. \cite{eynard2012multimodal}  proposed constructing multimodal spectral geometry by finding a common orthonormal basis $\hat{\bb{\Phi}}$ that approximately jointly diagonalizes the symmetric Laplacians $\bb{L}_k$ by solving the optimization problem  
\begin{eqnarray}
J(\Ll_1, \Ll_2) = \min_{ \hat{\bb{\Phi}} \in \RR^{n\times n}  } \sum_{k=1}^2\mathrm{off}(\hat{\bb{\Phi}}^\Tr \Ll_k \hat{\bb{\Phi}}) \,\,\,\,\, \text{s.t.} \,\,\,\,\, \hat{\bb{\Phi}}^\Tr \hat{\bb{\Phi}} = \bb{I}, 
\label{eq:jade}
\end{eqnarray}
where $\mathrm{off}(\Aa) = \sum_{i\neq j} a_{ij}^2$ denotes the squared norm of the off-diagonal elements of a matrix. 
Minimization of~(\ref{eq:jade}) can be carried out using a Jacobi-type method referred to as JADE \cite{Cardoso96jacobiangles}. 
Kovnatsky et al. \cite{KovBBKK:2013:EG} proposed a more efficient approach for finding the first few joint approximate eigenvectors representing $\hat{\bb{\phi}}_1, \hdots \hat{\bb{\phi}}_m$ as linear combinations of the eigenvectors $\bb{\phi}^1_1, \hdots \bb{\phi}^1_m$ and $\bb{\phi}^2_1, \hdots \bb{\phi}^2_m$ of $\bb{L}_1, \bb{L}_2$.

The joint basis $\hat{\bb{\Phi}}$ obtained in this way approximately diagonalizes the Laplacians, such that 
$\hat{\bb{\Phi}}^\Tr \bb{L}_k \hat{\bb{\Phi}} \approx \mathrm{diag}(\hat{\lambda}_{k,1}, \hdots, \hat{\lambda}_{k,n})$. 
The approximate matrices 
$$\hat{\Ll}_k = \hat{\bb{\Phi}}\mathrm{Diag}(\hat{\bb{\Phi}}^\Tr \Ll_k \hat{\bb{\Phi}}) \hat{\bb{\Phi}}^\Tr \approx \Ll_k, $$ 
obtained by setting to zero the off-diagonal elements of $\hat{\bb{\Phi}}^\Tr \bb{L}_k \hat{\bb{\Phi}}$ 
are jointly diagonalizable.   
Eynard et al. \cite{eynard2012multimodal} used the approximate joint eigenvectors $\hat{\bb{\Phi}}$ and the average joint approximate eigenvalues $\frac{1}{2}\sum_{k=1}^2 \hat{\lambda}_{k,i}$ 
to construct joint `heat kernels'  
\begin{eqnarray}
\hat{\bb{H}}^t_k = \hat{\mathbf{\Phi}} \, \frac{1}{2}\sum_{i=1}^2 \mathrm{diag}(e^{-t \hat{\lambda}_{k,1}}, \hdots, e^{-t \hat{\lambda}_{k,n}}) \hat{\mathbf{\Phi}}^\Tr,
\label{eq:jade_hk}
\end{eqnarray}
and 
multimodal diffusion distances 
\begin{eqnarray}
\hat{d}_t(x_p,x_q) = \left( \sum_{i=1}^n e^{-t \sum_{k=1}^2\hat{\lambda}_{k,i} } \hat{\phi}_{pi} \hat{\phi}_{qi} \right)^{1/2}.   
\label{eq:jade_diffdist}
\end{eqnarray}

%

\subsection{Relation between joint diagonalizability and commutativity}
Joint diagonalizability of matrices is intimately related to their commutativity. It is well-known that two symmetric matrices $\Aa,\Bb$ are jointly diagonalizable iff they commute 
\cite{horn1990matrix}.  
In \cite{Glashoff:2013uq}, we extended this result to the approximate setting, showing that almost jointly diagonalizable matrices almost commute:


\begin{theo}[\bf Glashoff-Bronstein 2013] 
There exist functions $\delta_1(x), \delta_2(x)$ satisfying $\lim_{x\rightarrow 0} \delta_i (x) = 0$, $i=1,2$, such that for any two symmetric $n\times n$ matrices $\Aa, \Bb$  with $\| \Aa\|_\mathrm{F} = \| \Bb\|_\mathrm{F} = 1$, 
$$\delta_1( \| \Aa\Bb - \Bb\Aa \|_\mathrm{F} )  \leq J(\Aa,\Bb) \leq n \delta_2( \| \Aa\Bb - \Bb\Aa \|_\mathrm{F} ). $$ 
Furthermore, the lower bound is tight. 
\end{theo}

On the other hand, it is known that almost commuting matrices are close to commuting matrices, e.g. in the following sense 
\cite{Huang_Lin,Rordam1996,Loring_Sorensen2010}:

\begin{theo}[\bf Lin 1997]
There exists a function $\epsilon(\delta)$ satisfying $\lim_{\delta \rightarrow 0}\epsilon(\delta) = 0$ with the following property: If $\Aa, \Bb$ are two self-adjoint $n\times n$ matrices satisfying $\|\ Aa\|_2, \| \Bb\|_2 \leq 1$, and $\| [\Aa,\Bb]\|_2 \leq \delta$, then there exists a pair $\tilde{\Aa'},\tilde{ \Bb}'$ of \emph{commuting matrices} satisfying $\| \Aa - \tilde{\Aa'} \|_2 \leq \epsilon(\delta)$ and $\|\Bb -\tilde{\Bb'} \|_2 \leq \epsilon(\delta)$.
\end{theo}
%

The combination of Theorems 2.1 and 2.2 implies that approximately jointly diagonalizable matrices are close to jointly diagonalizable matrices, and provides for an alternative to the joint diagonalization approaches used in \cite{eynard2012multimodal,KovBBKK:2013:EG}: instead of trying to approximately diagonalize the matrices $\Aa, \Bb$, we {\em minimally modify} $\Aa, \Bb$ to make them commute and thus become jointly diagonalizable,
\begin{eqnarray}
C(\Aa,\Bb) &=& \displaystyle \min_{ \tilde{\Aa}, \tilde{\Bb} }  
\| \tilde{\Aa} - \Aa \|_\mathrm{F}^2  + \| \tilde{\Bb} - \Bb \|_\mathrm{F}^2 
\,\,\, \text{s.t.} \,\,\, \| \tilde{\Aa}\tilde{\Bb} - \tilde{\Bb}\tilde{\Aa} \|_\mathrm{F}^2 = 0. 
\label{eq:problem}
\end{eqnarray} 

\noindent Finally, the following result\footnote{
An analogous theorem for the related problem of almost normal complex matrices is presented in \cite{Higham_matrix:nearness}, 
where it is attributed to \cite{Gabriel1979} 
and \cite{Causey1964}. 
} provides an even stronger connection between problems~(\ref{eq:problem}) and~(\ref{eq:jade}):
\begin{theo}
Let $\Aa, \Bb$ be symmetric matrices. Then, 
$$C(\Aa,\Bb) = J(\Aa,\Bb).$$ 
\end{theo}

\begin{proof}
Let us denote 
\begin{eqnarray}
C(\Aa,\Bb,\Xx,\Yy) &=&  
\| \Aa-\Xx \|_\mathrm{F}^2  + \|\Bb-\Yy \|_\mathrm{F}^2; \nonumber \\
J(\Aa,\Bb,\Uu) &=& \|\Uu^\Tr\Aa\Uu-\text{Diag}(\Uu^*\Aa\Uu)\|^2 + \|\Uu^*\Bb\Uu-\text{Diag}(\Uu^\Tr\Bb\Uu)\|^2,\nonumber
\end{eqnarray}
where $\Xx, \Yy$ is a pair of commuting matrices, and $\Uu$ is a unitary matrix. 
Let $\hat{\Uu}$ be the joint approximate eigenbasis of $\Aa, \Bb$ such that
 $J(\Aa,\Bb,\hat{\Uu})=J(\Aa,\Bb)$. We further define 
 \begin{eqnarray*}
 \tilde{\Aa}&=&\hat\Uu\text{Diag}(\hat{\Uu}^\Tr \Aa\hat{\Uu})\hat{\Uu}^\Tr; \\
 \tilde{\Bb}&=&\hat\Uu\text{Diag}(\hat{\Uu}^\Tr \Bb\hat{\Uu})\hat{\Uu}^\Tr.
 \end{eqnarray*}
Using the fact that the Frobenius norm is invariant under unitary transformations, we get the following sequence of inequalities:
\begin{eqnarray}
C(\Aa,\Bb)&\leq&C(\Aa,\Bb,\tilde{\Aa},\tilde{\Bb})\nonumber\nonumber  \\
&=&\|\Aa-\hat\Uu\text{Diag}(\hat{\Uu}^\Tr\Aa\hat{\Uu})\hat{\Uu}^\Tr\|_\mathrm{F}^2+\|\Bb-\hat\Uu\text{Diag}(\hat{\Uu}^\Tr\Bb\hat{\Uu})\hat{\Uu}^\Tr\|_\mathrm{F}^2\nonumber \\
&=& \|\hat{\Uu}^\Tr\Aa\hat\Uu-\text{Diag}(\hat{\Uu}^\Tr\Aa\hat{\Uu})\|_\mathrm{F}^2+\|\hat{\Uu}^\Tr\Bb\hat\Uu-\text{Diag}(\hat{\Uu}^\Tr\Bb\hat{\Uu})\|_\mathrm{F}^2\nonumber \\
&=&J(\Aa,\Bb,\hat{\Uu}) 
=J(\Aa,\Bb).
\end{eqnarray}

Now suppose that $\tilde{\Aa}$ and $\tilde{\Bb}$ are the closest commuting matrices to $\Aa, \Bb$ such that  $C(\Aa,\Bb,\tilde{\Aa},\tilde{\Bb})=C(\Aa,\Bb)$. 
Commuting matrices $\tilde{\Aa}, \tilde{\Bb}$ are jointly diagonalizable by a unitary matrix that we denote by $\tilde{\Uu}$. Since changing a zero-term in a matrix to a non-zero term can only increase the Frobenius norm, we get
\begin{eqnarray}
J(\Aa,\Bb)&\leq&J(\Aa,\Bb,\tilde{\Uu})\nonumber \\
&=& \|\tilde{\Uu}^\Tr\Aa\tilde\Uu-\text{Diag}(\tilde{\Uu}^\Tr\Aa\tilde{\Uu})\|_\mathrm{F}^2+\|\tilde{\Uu}^\Tr\Bb\tilde\Uu-\text{Diag}(\tilde{\Uu}^\Tr\Bb\tilde{\Uu})\|_\mathrm{F}^2\nonumber \\
&\leq& \|\tilde{\Uu}^\Tr\Aa\tilde\Uu-\text{Diag}(\tilde{\Uu}^\Tr\tilde{\Aa}\tilde{\Uu})\|_\mathrm{F}^2+\|\tilde{\Uu}^\Tr\Bb\tilde\Uu-\text{Diag}(\tilde{\Uu}^\Tr\tilde{\Bb}\tilde{\Uu})\|_\mathrm{F}^2\nonumber \\
&=&\|\Aa-\tilde\Uu\text{Diag}(\tilde{\Uu}^\Tr\tilde{\Aa}\tilde{\Uu})\tilde{\Uu}^\Tr\|_\mathrm{F}^2+\|\Bb-\tilde\Uu\text{Diag}(\tilde{\Uu}^\Tr\tilde{\Bb}\tilde{\Uu})\tilde{\Uu}^\Tr\|_\mathrm{F}^2\nonumber \\
&=&\|\Aa-\tilde{\Aa}\|_\mathrm{F}^2+\|\Bb-\tilde{\Bb}\|_\mathrm{F}^2 =C(\Aa,\Bb).
\label{}
\end{eqnarray}
\end{proof}

\noindent Because of $C(\Aa,\Bb)=J(\Aa,\Bb)$, all inequalities in the proof of Theorem~2.3 turn out to be equalities, so we immediately get the following  
%
%

\begin{corollary}
Let $\Aa,\Bb$ be symmetric matrices.

\noindent 1. Let $\hat{\bb{U}}$ be the approximate joint eigenbasis of $\Aa, \Bb$ such that $J(\Aa,\Bb, \hat{\Uu}) = J(\Aa,\Bb)$. 
Then, 
$\tilde{\Aa} = \hat{\bb{U}}\mathrm{Diag}(\hat{\bb{U}}^\Tr \Aa \hat{\bb{U}}) \hat{\bb{U}}^\Tr$ and  
$\tilde{\Bb} = \hat{\bb{U}}\mathrm{Diag}(\hat{\bb{U}}^\Tr \Bb \hat{\bb{U}}) \hat{\bb{U}}^\Tr$ 
are the closest commuting matrices to $\Aa,\Bb$ such that $C(\Aa,\Bb, \tilde{\Aa}, \tilde{\Bb}) = C(\Aa,\Bb)$.

\noindent 2. Let $\tilde{\Aa},\tilde{\Bb}$ be the closest commuting matrices such that $C(\Aa,\Bb, \tilde{\Aa}, \tilde{\Bb}) = C(\Aa,\Bb)$.  
Then, their joint eigenbasis $\hat{\Uu}$ satisfied $J(\Aa,\Bb, \hat{\Uu}) = J(\Aa,\Bb)$. 
\end{corollary}

\noindent In other words, the joint approximate diagonalization problem~(\ref{eq:jade}) and the closest commuting matrices problem~(\ref{eq:problem})  are {\em equivalent}, and we can solve one by solving the other. 
However, the big advantage of~(\ref{eq:problem}) is that we have explicit control over the structure of the resulting matrices $\tilde{\Aa}, \tilde{\Bb}$, while in~(\ref{eq:jade}) this is impossible.  
In particular, when applied to Laplacian matrices, we cannot guarantee that the matrices 
$\hat{\Ll}_k = \hat{\bb{\Phi}}\mathrm{Diag}(\hat{\bb{\Phi}}^\Tr \Ll_k \hat{\bb{\Phi}}) \hat{\bb{\Phi}}^\Tr$ obtained by approximate diagonalization of $\Ll_k$ are legal Laplacians (see Figure~\ref{fig:sparsity}).

In the following section, we solve problem~(\ref{eq:problem}) on the subset of Laplacian matrices and explore its application to the construction of multimodal spectral geometry.

\begin{figure}
\center{
\begin{overpic}
  [width=1\linewidth]{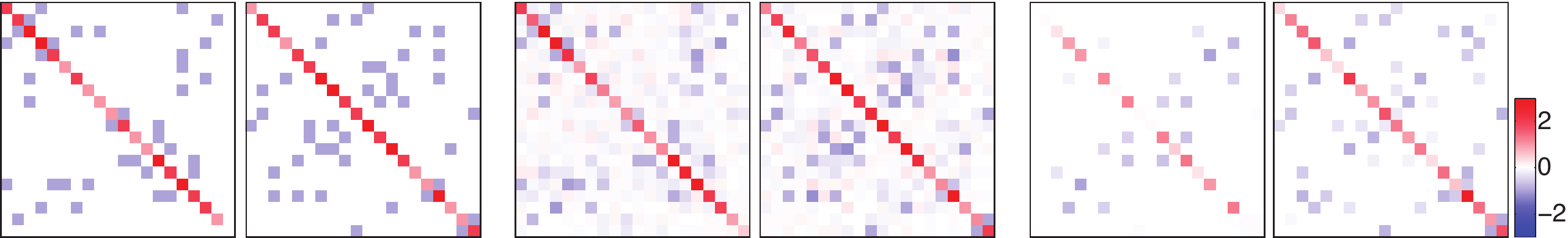}
\put(8,-3){\footnotesize $\Ll_1$ }  
\put(23,-3){\footnotesize $\Ll_2$ }  
\put(40,-3){\footnotesize $\hat{\Ll}_1$ }  
\put(55,-3){\footnotesize $\hat{\Ll}_2$ }  
\put(73,-3){\footnotesize $\tilde{\Ll}_1$ }  
\put(88,-3){\footnotesize $\tilde{\Ll}_2$ }  
\put(11,16){\footnotesize Original }  
\put(44,16){\footnotesize JADE}  
\put(78,16){\footnotesize CCO}  
\end{overpic}
  \caption{\label{fig:sparsity}  \small Comparison of the result of joint diagonalization (center) and closest commuting operator (right) problems applied to a pair of Laplacians (left). JADE does not preserve the sparse structure of the Laplacians. Even worse, matrices $\hat{\Ll}_k$ are not legal Laplacians as the sum of their rows is not zero anymore. } 
}
\end{figure}
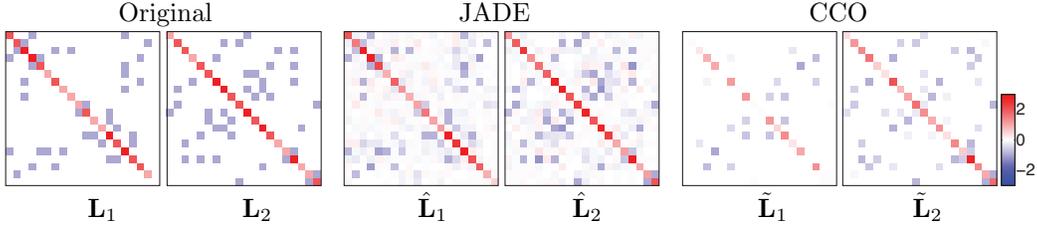

\section{Problem formulation}

%
%

Denote by $L(V,E)$  
the set of Laplacian matrices of an undirected graph $(V,E)$ with arbitrary edge weights. 
Let us be given two undirected graphs $G_k = (V,E_k),\, k = 1, 2$ with adjacency matrices $\Ww_k$  and Laplacians $\Ll_k$,  
%
%
%
%
%
and let  $\tilde{G}_k = (V,\tilde{E}_k)$ be new graphs, where 
the edges $\tilde{E}_k$ are defined either as 
$\tilde{E}_k = E_k$ (the connectivity of $\tilde{G}_k$ is identical to that of $G_k$), or as 
$\tilde{E}_k = \bigcup_{k=1}^2 E_k$ (the connectivity of $\tilde{G}_k$ is a union of the edge sets of $G_1$ and $G_2$). 
We denote their respective adjacency matrices by  $\tilde{\Ww}_k$ and the Laplacians  by $\tilde{\Ll}_k$.

We are looking for such edge weights that $\tilde{\Ll}_1  \in L(V,\tilde{E}_1)$ and $\tilde{\Ll}_2 \in L(V,\tilde{E}_2)$ {\em commute} and are {\em as close as possible} to $\Ll_1, \Ll_2$: 
\begin{eqnarray}
\label{eq:cost_pre}
C_L(\Ll_1, \Ll_2) & = & 
\displaystyle \min_{ \tilde{\Ll}_k \in L(V,\tilde{E}_k) }  
\sum_{k=1}^2 \| \tilde{\Ll}_k - \Ll_k \|_\mathrm{F}^2 
\,\,\, \text{s.t.} \,\,\, \| [\tilde{\Ll}_1, \tilde{\Ll}_2] \|_\mathrm{F}^2 = 0. 
\end{eqnarray}
Problem~(\ref{eq:cost_pre}) is a version of problem~(\ref{eq:problem}) where the space of the matrices is restricted to valid Laplacians with the same structure as $\bb{\Ll}_1, \bb{\Ll}_2$.
We call the Laplacians $\tilde{\Ll}_1, \tilde{\Ll}_2$ produced by solving~(\ref{eq:cost_pre}) the {\em closest commuting operators} (CCO).

Since $\tilde{\Ll}_1, \tilde{\Ll}_2$ commute, they are jointly diagonalizable, i.e., we can find a single eigenbasis $\tilde{\bb{\Phi}}$ such that $\tilde{\bb{\Phi}}^\Tr \tilde{\bb{L}}_k \tilde{\bb{\Phi}} = \tilde{\bb{\Lambda}}_k =  \mathrm{diag}(\tilde{\lambda}_{k,1}, \hdots, \tilde{\lambda}_{k,n})$.\footnote{Individual diagonalization of $\tilde{\Ll}_k$ does not guarantee that the respective eigenvectors are identical, as the eigenvectors are defined up to a sign (for matrices with simple spectrum), or more generally, up to an isometry in the eigen sub-spaces corresponding to eigenvalues with multiplicity greater than $1$. It therefore makes sense to jointly diagonalize $\tilde{\Ll}_k$ using e.g. JADE even in this case, see \cite{Bunse-Gerstner:1993}.}
W.r.t. to this eigenbasis, we can write the heat operators
\begin{eqnarray}
\tilde{\bb{H}}_k^t = e^{-t \tilde{\Ll}_k} = \tilde{\bb{\Phi}} e^{-t \tilde{\bb{\Lambda}}_k} \tilde{\bb{\Phi}}^\Tr,  
\label{eq:hk_}
\end{eqnarray}
and diffusion distances 
\begin{eqnarray}
\tilde{d}_{k,t}(x_p,x_q) = \left( \sum_{i=1}^n e^{-2t \lambda_{k,i} } \tilde{\phi}_{pi} \tilde{\phi}_{qi} \right)^{1/2}.    
\label{eq:diffdist_c}
\end{eqnarray}
%


\subsection{Existence of CCOs}

An important question is how far the CCOs $\tilde{\bb{L}}_1, \tilde{\bb{L}}_2$ can be from the original Laplacians $\Ll_1, \Ll_2$?
We should stress that Lin's Theorem 2.2 is not directly applicable to our problem~(\ref{eq:cost_pre}): it guarantees that if $\| \Ll_1\Ll_2 - \Ll_2 \Ll_1\|_\mathrm{F} \leq \epsilon$,  there exist two {\em arbitrary} 
matrices $\delta(\epsilon)$-close to $\Ll_1, \Ll_2$, while we are looking for two {\em Laplacians with the same structure}. 
The question is therefore whether there exists a version of Theorem 2.2 that holds for a subset of such matrices.

While answering this question is a subject for future theoretical research, we provide empirical evidence that almost-commuting Laplacians are close to commuting Laplacians. 
In our experiment shown in Figure~\ref{fig:random}, we generated $1270$ pairs of random Laplacian matrices of sizes $n=10, 15, 20, 25, 30, 40$ and $50$, with random $K$-neighbor connectivity ($K$ random per vertex, ranging between $1$ and $10$) and weights uniformly distributed in the interval $[0,1]$. 
We consider  two matrices `numerically commuting' if the Frobenius norm of their commutator is below $10^{-7}n$. 
The behavior observed in Figure~\ref{fig:random} suggests the following 
\begin{conjecture}
There exists a function $\delta(\epsilon)$ satisfying $\lim_{\epsilon\rightarrow 0}\delta(\epsilon) = 0$, such that 
$$
C_L(\Ll_1, \Ll_2) \leq \delta(J(\Ll_1, \Ll_2)). 
$$
\end{conjecture}
\noindent Stated differently, from Theorem 2.3 we know that $C_L(\Ll_1,\Ll_2) \geq C(\Ll_1,\Ll_2) = J(\Ll_1,\Ll_2)$. We conjecture that if the Laplacians $\Ll_1, \Ll_2$ almost commute, then $C_L(\Ll_1,\Ll_2)$ is close to $C(\Ll_1,\Ll_2)$. 

A counterexample to Conjecture~3.1  would be a point in Figure~\ref{fig:random} with small x-coordinate and large y-coordinate, which is not observed in our experiments.  
We leave the theoretical justification of this conjecture (or its disproval) for future work.

%
%
%

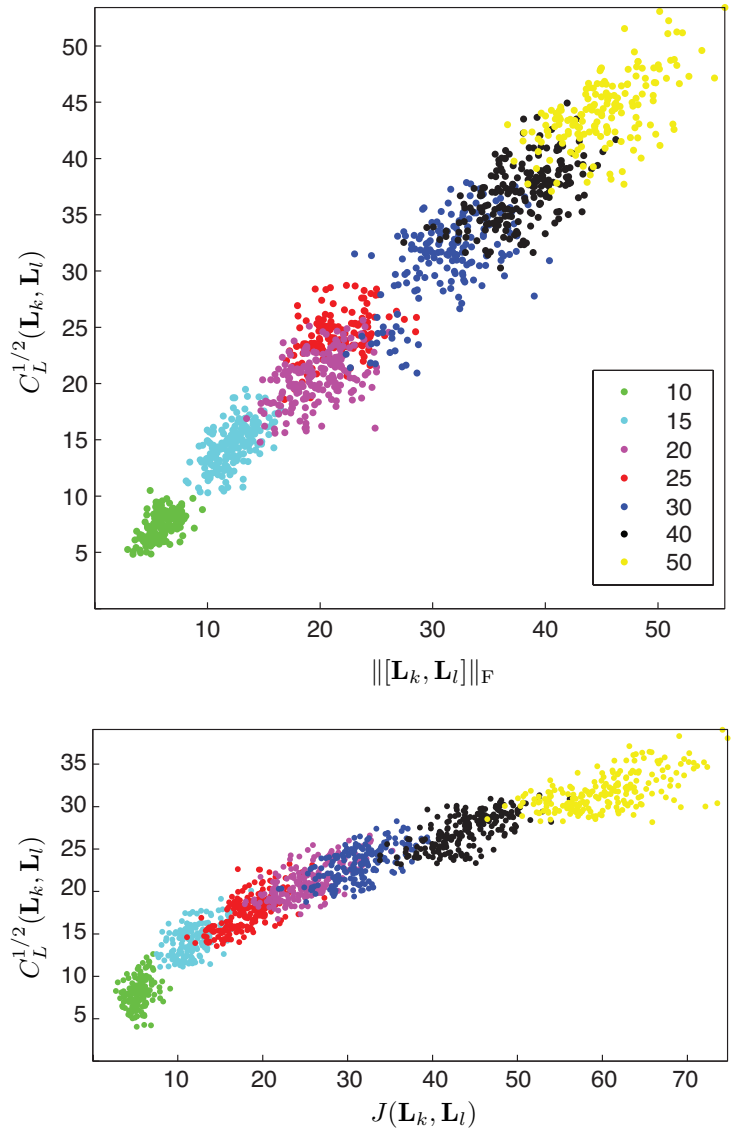
\begin{figure}
\center{
\begin{overpic}
  [width=0.65\linewidth]{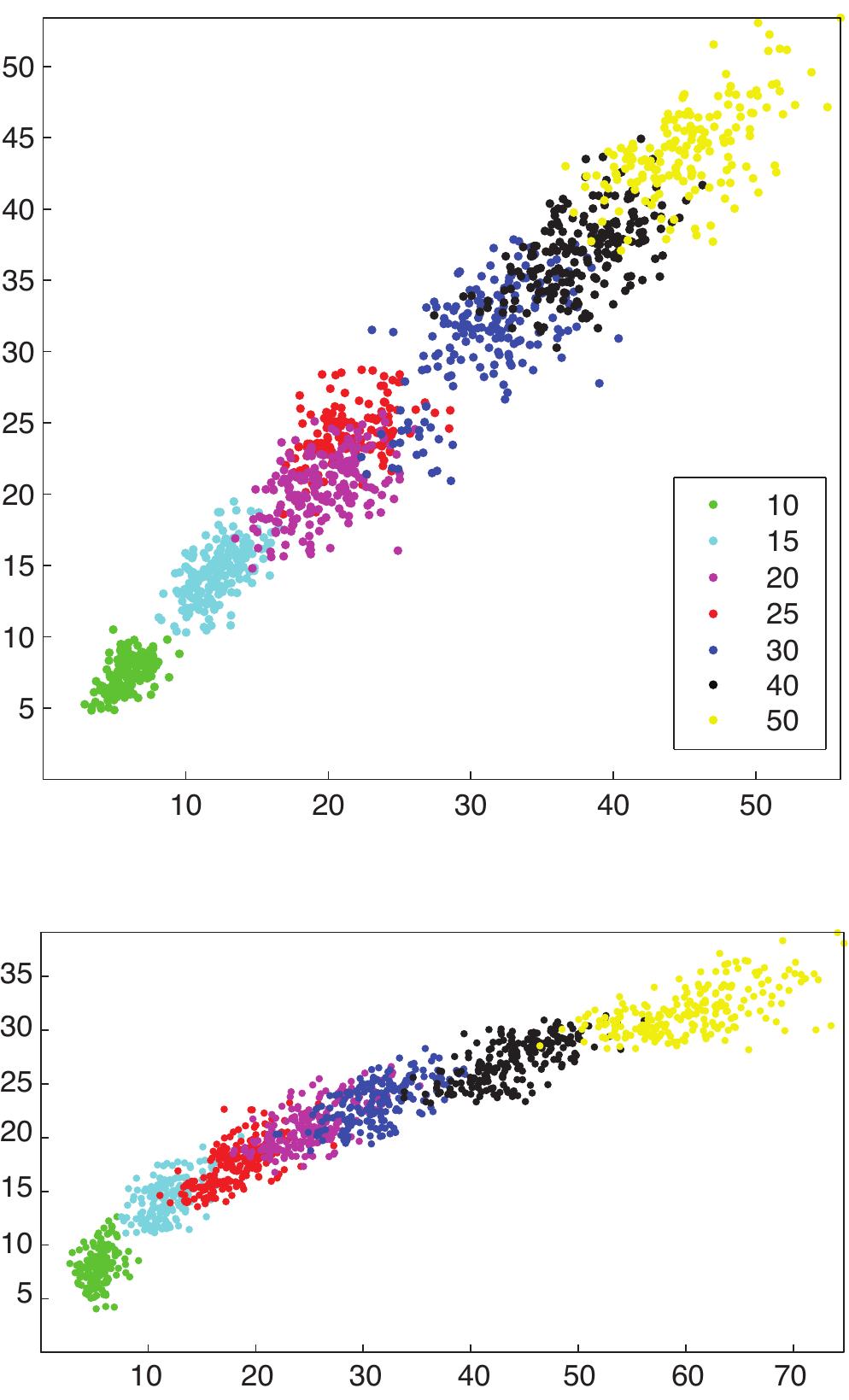}
\put(28,-2){\footnotesize $J(\Ll_k, \Ll_l)$ }
\put(28,38){\footnotesize $\| [\Ll_k, \Ll_l] \|_\mathrm{F}$ }
\put(-4.5,10){\footnotesize 
\rotatebox{90}{
$C_L^{1/2}(\Ll_k, \Ll_l)$
}
}  
\put(-4.5,63){\footnotesize 
\rotatebox{90}{
$C_L^{1/2}(\Ll_k, \Ll_l)$
}
} 
\end{overpic}\vspace{3mm}
  \caption{\label{fig:random}  \small Numerical evidence that almost-commuting Laplacians are close to commuting Laplacians, obtained on random graphs of different size (shown by color) and connectivity.  
  }
}
\end{figure}

\section{Numerical optimization}

In problem~(\ref{eq:cost_pre}), we are looking for new Laplacians $\tilde{\Ll}_k \in L(V,\tilde{E}_k)$. 
Let us denote by 
$\uu_k = (u^k_1, \hdots, u^k_{\tilde{M}_k})$ the edge weights of the new graphs (here, due to symmetry, $\tilde{M}_k = \frac{1}{2}|\tilde{E}_k|$).  
We parametrize the new adjacency matrices $\tilde{\Ww}_k(\uu_k)$ as 
\begin{eqnarray}
\tilde{w}^k_{ij}(\uu_k) &=& \left\{
\begin{array}{cc}
u_l^k & (i_l,j_l) \,\, \text{or} \,\, (j_l,i_l) \in \tilde{E}_k\\
0 & \text{else}.
\end{array}
\right. 
\end{eqnarray}
%
%
Then, we can rewrite~(\ref{eq:cost_pre}) as  
\begin{eqnarray}
\displaystyle \min_{ \{ \uu_k \geq 0
\}_{k=1}^2 }  &&
\sum_{k=1}^2 \| \tilde{\Ll}_k(\uu_k) - \Ll_k \|_\mathrm{F}^2 
\,\,\, \text{s.t.} \,\,\, \| [\tilde{\Ll}_1(\uu_1), \tilde{\Ll}_2(\uu_2)] \|_\mathrm{F}^2 = 0. 
\label{eq:cost_}
\end{eqnarray}
%
%
%
In practice, it is more convenient to solve an unconstrained formulation of problem~(\ref{eq:cost_}), 
\begin{eqnarray}
\displaystyle \min_{ \{ \uu_k \geq 0
\}_{k=1}^2 }  &&
\sum_{k=1}^2 \| \tilde{\Ll}_k(\uu_k) - \Ll_k \|_\mathrm{F}^2 + \alpha \| [\tilde{\Ll}_1(\uu_1), \tilde{\Ll}_2(\uu_2)] \|_\mathrm{F}^2, 
\label{eq:cost}
\end{eqnarray}
where $\alpha>0$ is a weight parameter.

The solution of problem~(\ref{eq:cost}) is carried out using standard first-order optimization techniques, where to ensure that we obtain a legal Laplacian, the weights are projected onto the interval $[0,1]$ after each iteration. 

We differentiate the cost function in~(\ref{eq:cost}) w.r.t. the elements of the matrices $\tilde{\bb{W}}_k$, out of which only the relevant elements $\tilde{w}^k_{ij} : (i,j) \in \tilde{E}_k$ are used. 
The derivative of the distance terms in~(\ref{eq:cost}) are given by
$$
\frac{\partial}{\partial \tilde{w}^k_{ij}} \| \Ll_k -  \tilde{\Ll}_k \|_\mathrm{F}^2 
= 2\left( \bb{O} + \Ll_k
- \tilde{\Ll}_k, 
\right)_{ij}, 
$$
where $\bb{O}$ is an $n\times n$ matrix with equal columns $\bb{O} = ( \mathrm{diag}(\Ll_k), \hdots,   \mathrm{diag}(\Ll_k) )$ (in our notation, $\mathrm{diag(\Aa)}$ is a column vector containing the diagonal elements of $\Aa$). 
The derivative of the commutator term w.r.t. the elements of the matrix $\tilde{\bb{W}}_1$ is given by 
$$
\frac{\partial}{\partial \tilde{w}^1_{ij}} \| \tilde{\Ll}_1 \tilde{\Ll}_2 - \tilde{\Ll}_2 \tilde{\Ll}_1 \|_\mathrm{F}^2 = 2 \left (  
\bb{O}_1 - \tilde{\Ll}_2^\Tr(\tilde{\Ll}_2\tilde{\Ll}_1 - \tilde{\Ll}_1\tilde{\Ll}_2) 
- \bb{O}_2 + (\tilde{\Ll}_2\tilde{\Ll}_1 - \tilde{\Ll}_1\tilde{\Ll}_2)\tilde{\Ll}_2^\Tr
\right)_{ij},
$$
where $\bb{O}_k$ are $n \times n$ matrices with equal columns given by 
\begin{eqnarray*}
\bb{O}_1 &=& ( \mathrm{diag}(\tilde{\Ll}_2^\Tr(\tilde{\Ll}_2\tilde{\Ll}_1 - \tilde{\Ll}_1\tilde{\Ll}_2)), \hdots,   \mathrm{diag}(\tilde{\Ll}_2^\Tr(\tilde{\Ll}_2\tilde{\Ll}_1 - \tilde{\Ll}_1\tilde{\Ll}_2)) ),\\
\bb{O}_2 &=& ( \mathrm{diag}((\tilde{\Ll}_2\tilde{\Ll}_1 - \tilde{\Ll}_1\tilde{\Ll}_2)\tilde{\Ll}_2^\Tr), \hdots,   \mathrm{diag}((\tilde{\Ll}_2\tilde{\Ll}_1 - \tilde{\Ll}_1\tilde{\Ll}_2)\tilde{\Ll}_2^\Tr) ).
\end{eqnarray*}
By symmetry considerations, 
$$
\frac{\partial}{\partial \tilde{w}^2_{ij}} \| \tilde{\Ll}_1 \tilde{\Ll}_2 - \tilde{\Ll}_2 \tilde{\Ll}_1 \|_\mathrm{F}^2 = 
-\frac{\partial}{\partial \tilde{w}^1_{ij}} \| \tilde{\Ll}_1 \tilde{\Ll}_2 - \tilde{\Ll}_2 \tilde{\Ll}_1 \|_\mathrm{F}^2. 
$$

\section{Generalizations}

Our problem formulation~(\ref{eq:cost}) assumes that the two graphs have the same set of vertices $V$ and different edges $E_k$, 
having thus Laplacians $\Ll_k$ of equal size $n\times n$. 
%
A more general setting is of two graphs with different sets of vertices and edges, $G_k = (V_k, E_k)$, $|V_k| = n_k$, and 
the corresponding Laplacians of size  $n_k \times n_k$.

Our CCO problem can be extended to this setting using the notion of {\em functional correspondence} \cite{ovsjanikov2012functional}, expressed at an $n_2 \times n_1$ matrix $\bb{T}_{12}$ transferring functions defined on $V_1$ to $V_2$, and an $n_1 \times n_2$ matrix $\bb{T}_{21}$ going the other way around. 
In this setting, we can define an operator on the space of functions $L^2(V_1)$ by the composition $\bb{T}_{21}\tilde{\Ll}_2\bb{T}_{12}$ (or, equivalently, an operator on the space of functions $L^2(V_2)$ as $\bb{T}_{12}\tilde{\Ll}_1\bb{T}_{21}$). 
Our problem thus becomes
\begin{eqnarray}
\displaystyle \min_{ \{ \uu_k \geq 0
\}_{k=1}^2 }  
\sum_{k=1}^2 \| \tilde{\Ll}_k(\uu_k) - \Ll_k \|_\mathrm{F}^2 + \alpha \| [\tilde{\Ll}_1(\uu_1), \bb{T}_{21}\tilde{\Ll}_2(\uu_2)\bb{T}_{12}] \|_\mathrm{F}^2. 
\label{eq:cost_fun}
\end{eqnarray}
We call the term $[\tilde{\Ll}_1, \bb{T}_{21}\tilde{\Ll}_2\bb{T}_{12}]$ the {\em generalized commutator} of $\tilde{\Ll}_1$ and $\tilde{\Ll}_2$.

The functional correspondence $\bb{T}$ can be assumed to be given, or found from a set of corresponding vectors as proposed by Ovsjanikov et al. \cite{ovsjanikov2012functional}: given a set of functions $\bb{F} = (\bb{f}_1, \hdots, \bb{f}_q)$ on $V_1$ and corresponding functions  $\bb{G} = (\bb{g}_1, \hdots, \bb{g}_q)$ on $V_2$ (such that $\bb{T}_{12}\bb{f}_i = \bb{g}_i$), one can decompose $\bb{F}$ and $\bb{G}$ in the first $m$ eigenvectors $\bar{\bb{\Phi}}_k = (\bb{\phi}^k_1, \hdots, \bb{\phi}^k_m)$ of the corresponding Laplacians $\bb{L}_k$, yielding a system of $q$ equations with $m^2$ variables 
\begin{eqnarray}
\bb{C}  \bar{\bb{\Phi}}_1^\Tr \bb{F} =  \bar{\bb{\Phi}}_2^\Tr \bb{G},
\label{eq:fcorr}
\end{eqnarray}
where the $m\times m$ matrix $\bb{C}$ translates Fourier coefficients between the bases $\bar{\bb{\Phi}}_1$ and $\bar{\bb{\Phi}}_2$. 
The correspondence can be thus represented as $\bb{T}_{12} =  \bar{\bb{\Phi}}_2 \bb{C}  \bar{\bb{\Phi}}_1^\Tr$.
(A more general setting of finding the matrix $\bb{C}$ when the order of the columns of $\bb{F}, \bb{G}$ is unknown and outliers are present was discussed by Pokrass et al. \cite{pokrass2013sparse}).

\section{Results}

In this section, we demonstrate our CCO approach on several synthetic and real datasets coming from shape analysis, manifold learning, and pattern recognition problems. 
The leitmotif of all the experiments is, given two datasets representing similar objects in somewhat different ways, to reconcile the information of the two modalities producing a single consistent representation.

In all the experiments, we used unnormalized Laplacians (\ref{eq:laplacian}) constructed with Gaussian weights. 
Optimization of~(\ref{eq:cost}) was performed using conjugate gradients with inexact Armijo linesearch \cite{bertsekas1999nonlinear} with $\alpha$ in the range $10^4 - 10^8$. 
The edges $\tilde{E}_k = E_k$ were selected 
to preserve the connectivity of the original graphs.  
The information about the datasets as well as approximate timing (complexity of cost function and gradient evaluation, measured on a MacBook Air) is summarized in Table~\ref{tab:timing}.

\begin{table*}[htdp]\small
\begin{center}
\begin{tabular}{r cccc}

{\bf Dataset} & $n$  & $\tilde{M}_1$  & $\tilde{M}_2$  & T (sec)\\
\hline
{\em Caltech} 	& 105 & 791 & 678 & 0.0116\\
{\em Ring} 	& 140 & 149 & 149 & 0.0059\\
{\em Circles} 	& 195 & 443 & 446 & 0.0125\\
{\em Swissroll} 	& 400 & 866 & 877 & 0.0766\\
{\em Man} 		& 500 & 915 & 922 & 0.1195\\
{\em Reuters} & 600 & 10122 & 10669 & 1.2203\\
\hline
\end{tabular}\vspace{-2mm}
\end{center}
\caption{\label{tab:timing} \small Number of degrees of freedom and computational time of cost function and its gradient on different datasets. 
}
\end{table*}

{\bf Circles:} we used two graphs, shaped as four eccentric circles containing 195 points and having different connectivity (Figure~\ref{fig:circles}, left). 
The closest commuting Laplacians were found using the procedure described above and result in graph weights shown in Figure~\ref{fig:circles} (right): the optimization performs a `surgery' disconnecting the inconsistent connections and producing four connected components. 

{\bf Ring:} We used a ring and a cracked ring sampled at 140 points and connected using 4 nearest neighbors (Figure~\ref{fig:ring_evec}) to visualize the effect of our optimization on the resulting Laplacian eigenvectors. Figure~\ref{fig:ring_evec} (top) shows the first few eigenvectors $\bb{\Phi}_1, \bb{\Phi}_2$  
of the original Laplacians $\bb{L}_1, \bb{L}_2$: their structure differs dramatically.  
The CCO optimization cuts the connections in the first dataset, making the two rings topologically equivalent. Since the new Laplacians $\tilde{\bb{L}}_1, \tilde{\bb{L}}_2$ commute, they are jointly diagonalizable and thus the new sets of eigenvectors 
are identical ($\tilde{\bb{\Phi}}_1 = \tilde{\bb{\Phi}}_2 = \tilde{\bb{\Phi}}$, as shown in Figure~\ref{fig:ring_evec}, bottom). 

Figure~\ref{fig:ring_hk} shows the heat kernels computed on the original graphs ($\bb{H}^t_1, \bb{H}^t_2$, left) and after the optimization ($\tilde{\bb{H}}^t_1, \tilde{\bb{H}}^t_2$, right). 
For comparison, we also show the `joint' heat kernels $\hat{\bb{H}}^t$ obtained using joint diagonalization of the original Laplacians computed with JADE \cite{Cardoso96jacobiangles}. 
The latter is not a valid heat operator as it contains negative, albeit small, values (Figure~\ref{fig:ring_hk}, center).

{\bf Human shapes:} We used two poses of the human shape from the TOSCA dataset \cite{bronstein2008ngn}, uniformly sampled at 500 points and connected using 5 nearest neighbors. The resulting graphs have different topology (the hands are connected or disconnected, compare Figure~\ref{fig:man} top and bottom). 
We computed the heat diffusion distance with time parameter $t=100$ according to~(\ref{eq:diffdist}), truncating the sum after $100$ terms.  
Computing $d_t$ on the original graphs (Figure~\ref{fig:man}, left) manifests the difference in the graph topology: the distance from the fingers of the left hand to those of the right hand differs dramatically, as in one graph one has to go through the upper part of the body, while in the other one can `shortcut' across the hands connections. 
Our optimization disconnects these links (Figure~\ref{fig:man}, right) making the distance in both cases behave similarly. 
For comparison, we show the result of  simultaneous diagonalization using JADE (Figure~\ref{fig:man}, center), where the distance $\hat{d}_t$ is computed using joint approximate eigenvectors and average approximate joint eigenvalues as defined in~(\ref{eq:jade_diffdist}).

{\bf Swiss rolls:} We used two Swiss roll surfaces with slightly different embeddings and geometric noise, sampled at 400 points and connected using 4 nearest neighbors. Because of the different embeddings, the two graphs have different topology (the first one cylinder-like and the second one plane-like, see Figure~\ref{fig:swiss} top left). 
As a result, the embedding of the two Swiss rolls into the plane using Laplacian eigenmaps differ dramatically (Figure~\ref{fig:swiss}, bottom left). 

Performing our CCO optimization removes the topological noise making both graphs embeddable into the plane without self-intersections (Figure~\ref{fig:swiss}, top right). The resulting eigenmaps have correct topology and are perfectly aligned (bottom, right). 
For comparison, we show the joint diagonalization result (bottom, center). 

Finally, in Figure~\ref{fig:swiss} (bottom, right) we show optimization results obtained using a sparse set of pointwise correspondences from which a smooth functional correspondence was computed according to~(\ref{eq:fcorr}) and used in the generalized commutator in~(\ref{eq:cost_fun}).

{\bf Caltech:} We used the dataset from \cite{eynard2012multimodal}, containing 105 images belonging to 7 image classes (15 images per class) taken from the Caltech-101 dataset. 
The images were represented using the bio-inspired and the PHOW features used as two different modalities. 
We constructed the unnormalized Laplacian in each of the modalities using self-tuning weights, and computed the diffusion distance using the scale $t=1.6$ between all the images. 

Figure~\ref{fig:caltech} (left) shows the obtained diffusion distances. The CCO approach allows a significantly better distinction between image classes, which is manifested in higher ROC curves (Figure~\ref{fig:caltech}, right). 


{\bf Multiview clustering:}
We reproduce the multi-view clustering experiment from \cite{eynard2012multimodal}, wherein we use the previously described {\em Caltech} dataset; a subset of the \emph{NUS} dataset \cite{nus-wide-civr09} containing images (represented by 64-dimensional color histograms) and their text annotations (represented by 1000-dimensional bags of words); 
 the UCI \emph{Digits} dataset \cite{alpaydin1998cascading,liu2013multi} represented using 76 Fourier coefficients and the 240 pixel averages in $2\times 3$ windows; and the 
\emph{Reuters} dataset \cite{amini2010learning,liu2013multi} with the English and French languages used as two different modalities.
The goal of the experiment is to use the data in two modalities to obtain a multi-modal clustering that performs better than each single modality.

We use spectral clustering technique, consisting of first embedding the data in a low-dimensional space of the first eigenvectors, and then applying the standard K-means. The embedding is generated by the eigenvectors of each of the Laplacians individually ({\em unimodal}), by the approximate joint eigenvectors obtained by JADE, and the eigenvectors of the modified Laplacians produced by our CCO procedure. 
As a reference, we show the performance of the state-of-the-art Multimodal non-negative matrix factorization (MultiNMF) method \cite{liu2013multi} for multi-view clustering.
Table~\ref{tab:clustering1} shows the clustering performance of these different methods in terms of accuracy as defined in \cite{BekkermanJ07} and normalized mutual information (NMI).

\begin{figure}
\center{
\begin{overpic}
  [width=1\linewidth]{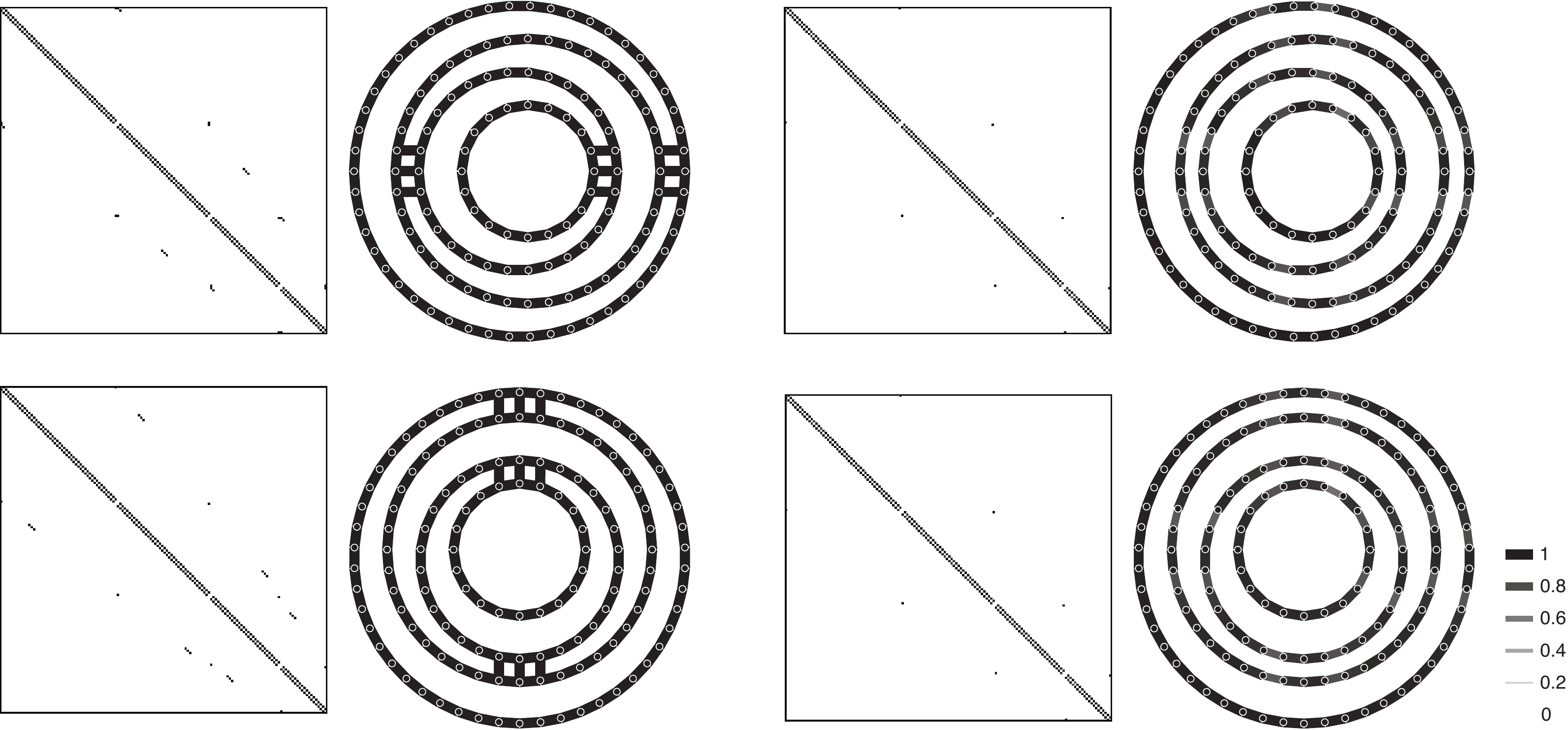}
    \put(20,-3){\footnotesize Original}
    \put(70,-3){\footnotesize COO}    
\end{overpic}\vspace{3mm}
  \caption{\label{fig:circles}  \small Graphs and adjacency matrices of the original data (left) and CCO (right).  
  Graph weights are shown with edge thickness and gray shades.  }
}
\end{figure}

\begin{figure}
\center{
\begin{overpic}
  [width=1\linewidth]{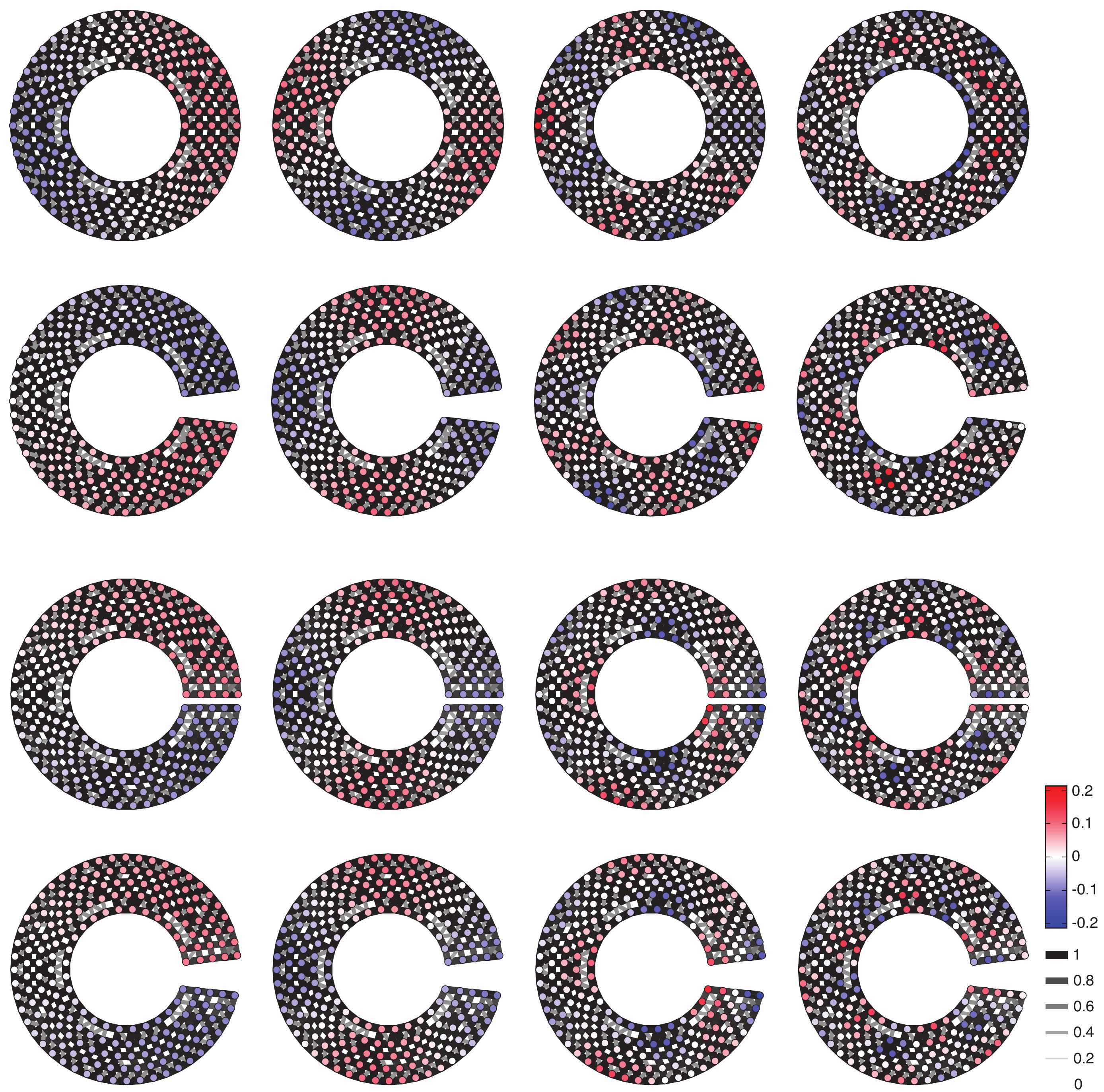}
    \put(10,74.5){\footnotesize $\bb{\phi}^1_2$}
    \put(34,74.5){\footnotesize $\bb{\phi}^1_5$}
    \put(58,74.5){\footnotesize $\bb{\phi}^1_{15}$}
    \put(82,74.5){\footnotesize $\bb{\phi}^1_{30}$}    
    \put(10,49.5){\footnotesize $\bb{\phi}^2_2$}
    \put(34,49.5){\footnotesize $\bb{\phi}^2_5$}
    \put(58,49.5){\footnotesize $\bb{\phi}^2_{15}$}
    \put(82,49.5){\footnotesize $\bb{\phi}^2_{30}$}        
    \put(10,22.5){\footnotesize $\tilde{\bb{\phi}}^1_2$}
    \put(34,22.5){\footnotesize $\tilde{\bb{\phi}}^1_5$}
    \put(58,22.5){\footnotesize $\tilde{\bb{\phi}}^1_{15}$}
    \put(82,22.5){\footnotesize $\tilde{\bb{\phi}}^1_{30}$}    
    \put(10,-2.5){\footnotesize $\tilde{\bb{\psi}}^2_2$}
    \put(34,-2.5){\footnotesize $\tilde{\bb{\psi}}^2_5$}
    \put(58,-2.5){\footnotesize $\tilde{\bb{\psi}}^2_{15}$}
    \put(82,-2.5){\footnotesize $\tilde{\bb{\psi}}^2_{30}$}           
\end{overpic}\vspace{3mm}
  \caption{\label{fig:ring_evec}  \small Eigenvectors of the original graph Laplacians ($\bb{\phi}^k_i$,  first and second rows) and the CCO ($\tilde{\bb{\phi}}^k_i$, third and fourth rows).  The eigenvectors of the CCO coincide, proving that they are jointly diagonalizable. 
  Graph weights are shown with edge thickness and gray shades. Eigenvector are shown with red-blue colormap.  }
}
\end{figure}

\begin{figure}
\center{
\begin{overpic}
  [width=0.9\linewidth]{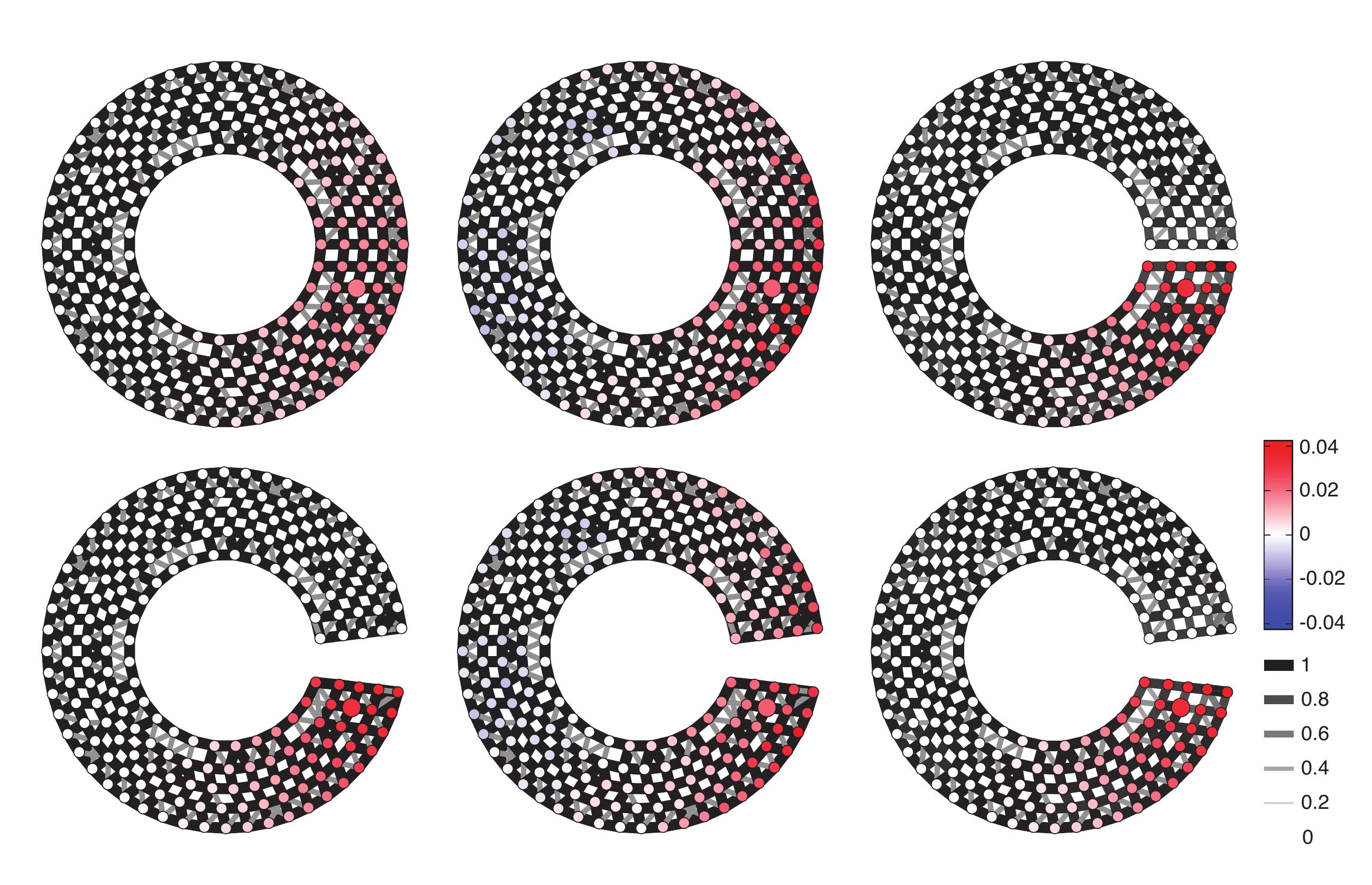}
    \put(12,0){\footnotesize Original}
    \put(43,0){\footnotesize JADE}
    \put(75,0){\footnotesize CCO}    
\end{overpic}\vspace{1mm}
  \caption{\label{fig:ring_hk}  \small Heat kernel at the point shown in big circle, computed using the original graph Laplacians ($\bb{H}_k^{20}$, left), joint diagonalization ($\hat{\bb{H}}_k^{20}$, middle), and CCO ($\tilde{\bb{H}}_k^{20}$, right).  
  Graph weights are shown with edge thickness and gray shades. Heat kernel values are shown with red-blue colormap. 
  JADE produces an invalid heat kernel, which has negative values. 
   }
}
\end{figure}

\begin{figure}
\center{
\begin{overpic}
  [width=1\linewidth]{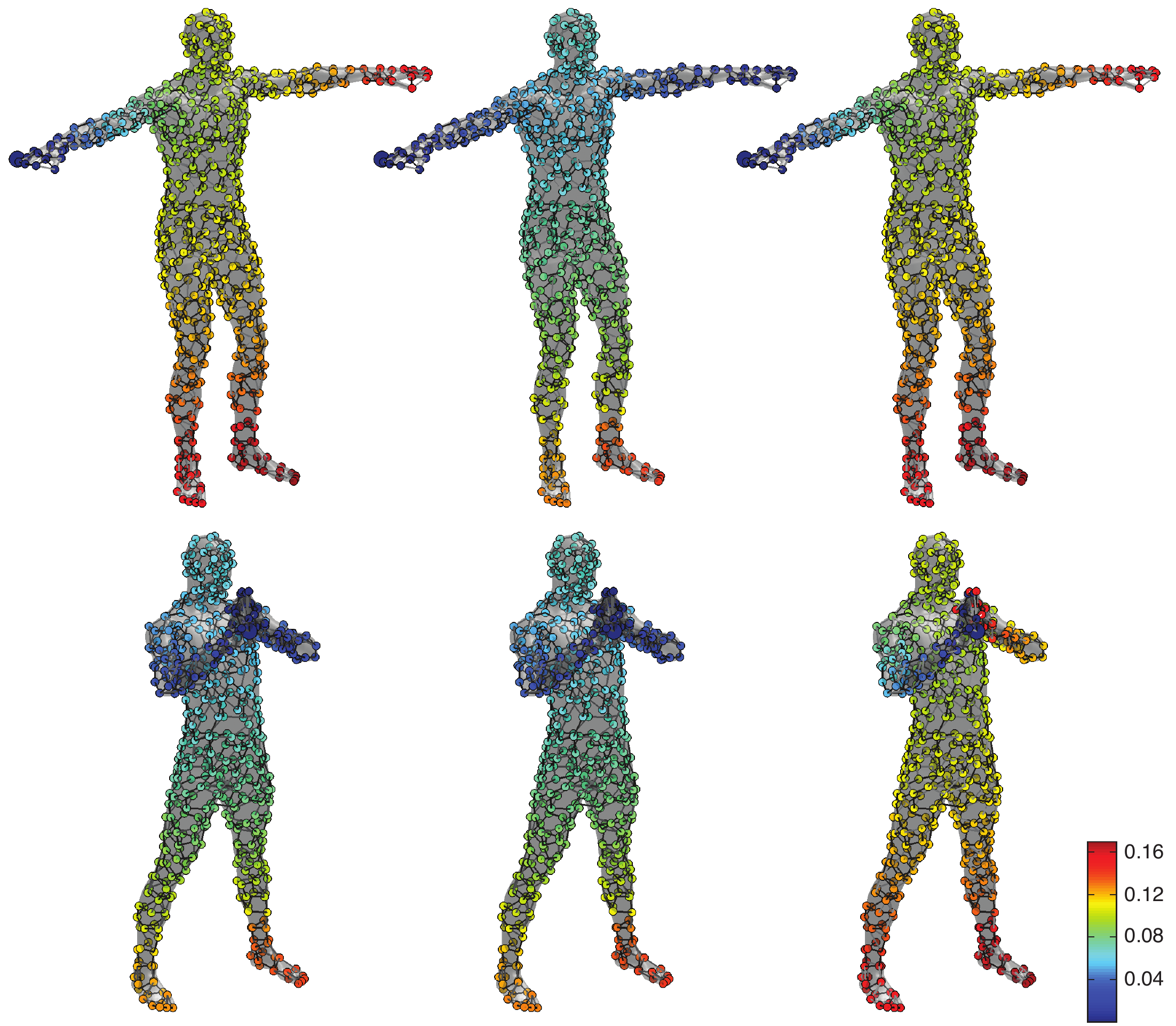}
    \put(15,-2){\footnotesize Original}
    \put(47,-2){\footnotesize JADE}
    \put(78,-2){\footnotesize CCO}    
\end{overpic}\vspace{3mm}
  \caption{\label{fig:man}  \small Diffusion distance from the point on the left hand (shown in big circle), computed using the original graph Laplacians (left) and CCO (right).  
}
}
\end{figure}

\begin{figure}
\center{
\begin{overpic}
  [width=1\linewidth]{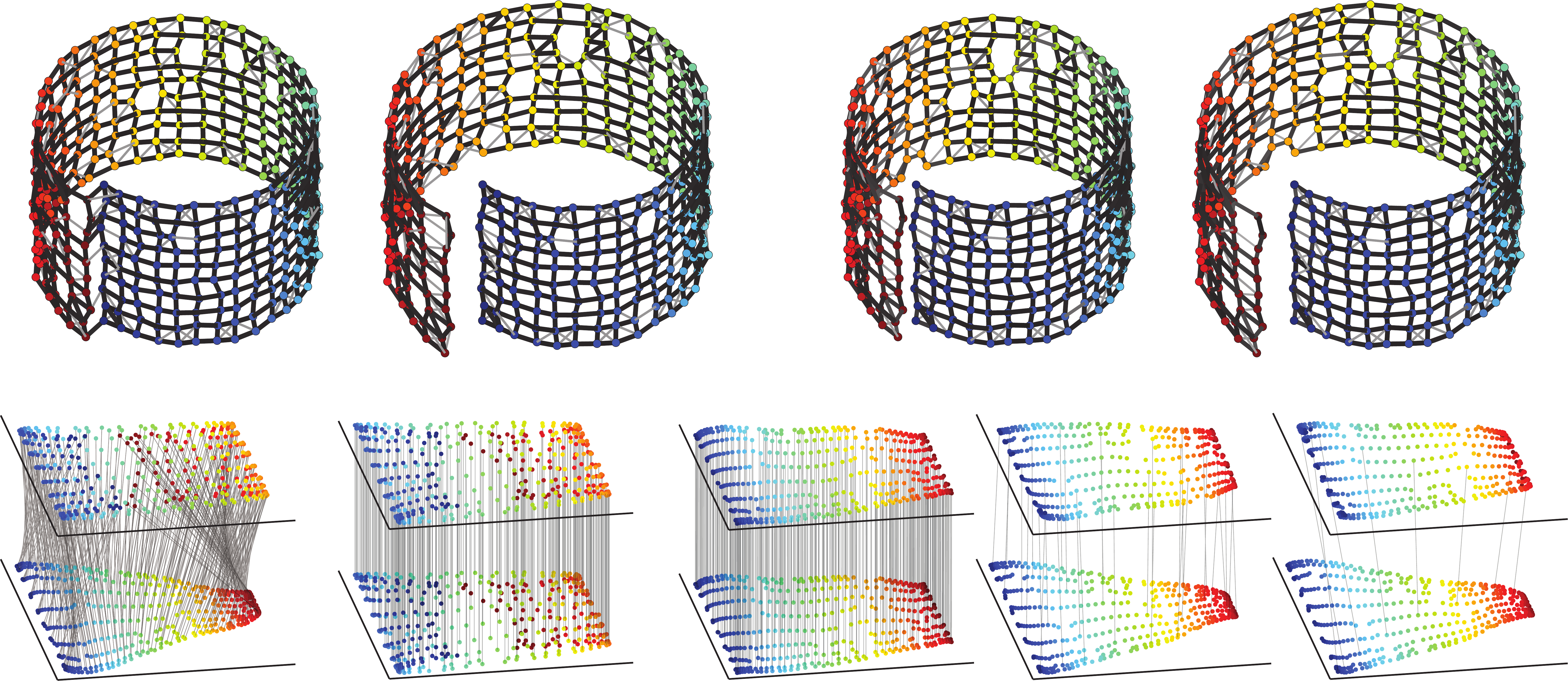}
    \put(17,19){\footnotesize Original}
    \put(73,19){\footnotesize CCO}    
    \put(6,-2.75){\footnotesize Original}        
    \put(28,-2.75){\footnotesize JADE}        
    \put(44,-2.75){\footnotesize CCO 100\% corr.}   
     \put(68,-2.75){\footnotesize 7.7\% corr.}      
     \put(87,-2.75){\footnotesize 2\% corr.}  
\end{overpic}\vspace{2mm}
  \caption{\label{fig:swiss}  \small First row: Swiss rolls with different connectivity before (left) and after (right)  optimization. 
Second row: Laplacian eigenmaps before (leftmost) and after optimization using different number of corresponding points (second to fourth column).  Color and lines show corresponding points.  }
}
\end{figure}

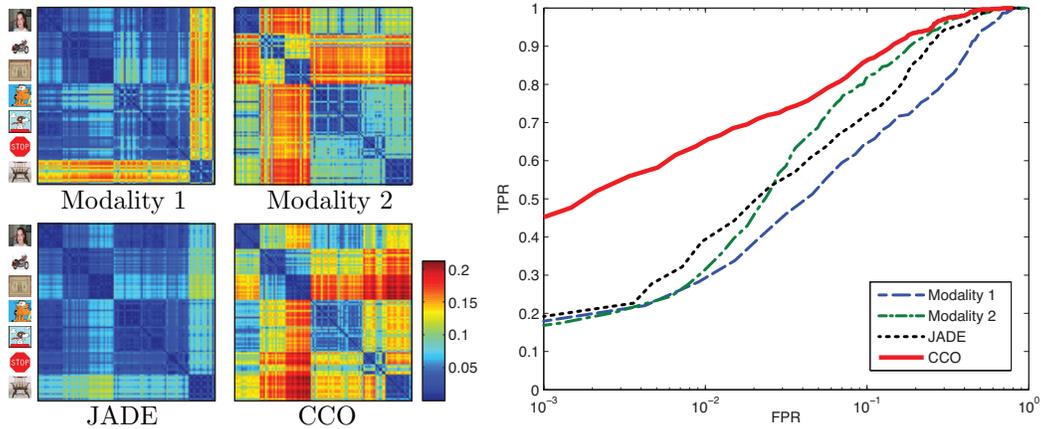
\begin{figure}
\center{
\begin{overpic}
  [width=1\linewidth]{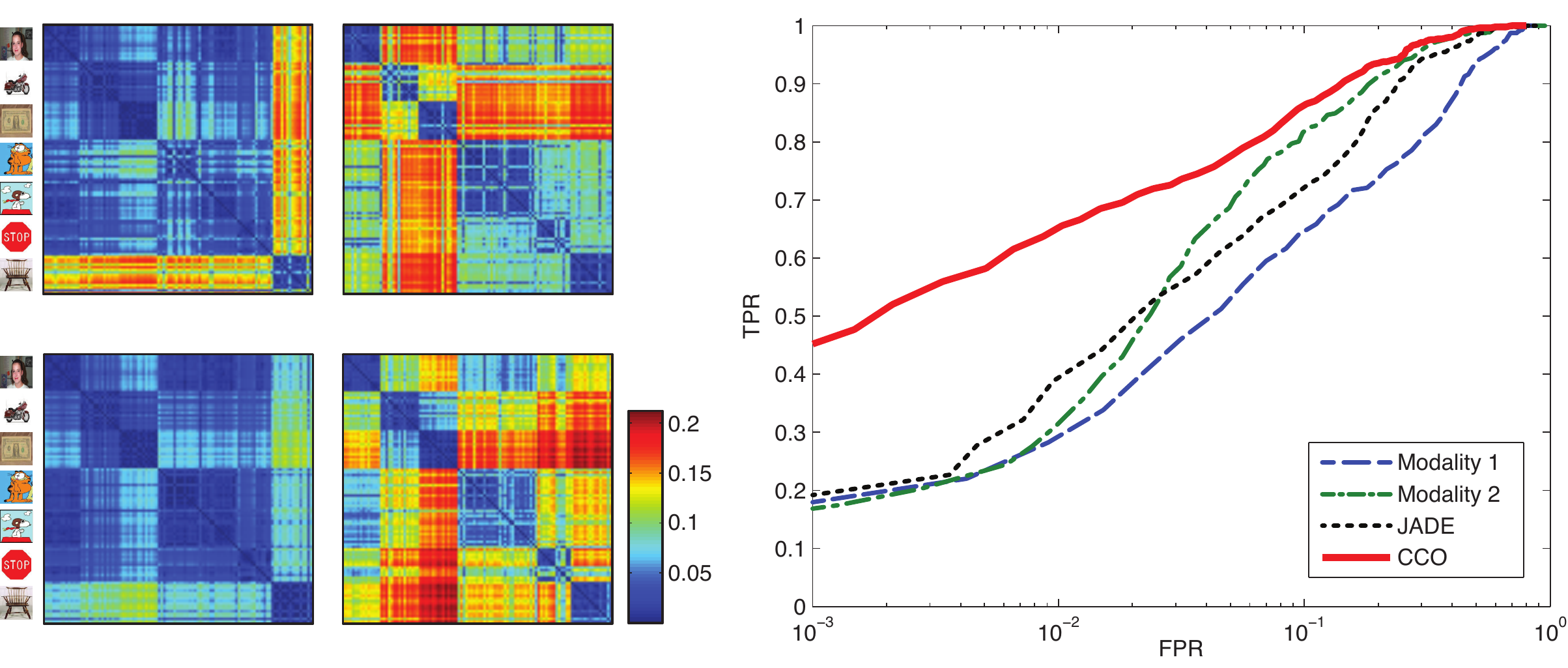}
    \put(5,21.5){\footnotesize Modality 1}
    \put(25,21.5){\footnotesize Modality 2}
    \put(7.5,0.5){\footnotesize JADE}
    \put(28,0.5){\footnotesize CCO}
\end{overpic}\vspace{0mm}
  \caption{\label{fig:caltech}  \small Left: diffusion distances computed on the Caltech dataset using independently the two modalities, and their combination with joint diagonalization and CCO method. 
  Right: ROC curves showing the tradeoff between false positive and true positive rates as function of a global threshold applied to the distance matrix (higher curves implies better discriminative power of the distance).
  }
}
\end{figure}

\begin{table*}[htdp]\small
\begin{center}
\begin{tabular}{rc | cccc}

& & \multicolumn{4}{c}{\bf Accuracy / NMI}\\
{\bf Dataset} & \hspace{-1mm}$n$\hspace{-1mm} &  Unimodal$^*$  &  MultiNMF 
& JADE 
& CCO \\
\hline
{\em Caltech} 
&	\hspace{-1mm}105\hspace{-1mm}	&	77.1 / 75.3	&	$**$			&	82.9 / 83.0	&	{\bf 90.5} / {\bf 93.4}	\\
{\em NUS} 
&	\hspace{-1mm}145\hspace{-1mm}	&	82.1 / 76.9	&	76.7 / 78.4	&	77.9 / 75.5	&	{\bf 86.9} / {\bf 84.4}	\\
{\em Reuters} 
&	\hspace{-1mm}600\hspace{-1mm}	&	{\bf 58.8} / 41.0	&	53.1 / 40.5	&	52.8 / 37.5	&	57.3 / {\bf 42.5}	\\
{\em Digits} 
&	\hspace{-1mm}2000\hspace{-1mm}	&	83.4 / 82.2	&	86.1	 / 78.1	&	84.5 / 84.0	&	{\bf 90.5} / {\bf 85.7}	\\
\hline
\end{tabular}\vspace{-2mm}
\end{center}
\caption{\label{tab:clustering1} \small Clustering performance (in $\%$) on four datasets. 
$^*$Best modality is shown.  $^{**}$Since Multi-NMF requires explicit coordinates of the data points, while {\em Caltech} data is represented implicitly as kernels, we could not measure its performance on this dataset.
}
\end{table*}

\section{Conclusions}

In this paper, we presented a novel approach for a principled construction of multimodal spectral geometry. 
Our approach is based on the observation that almost commuting matrices are close to commuting matrices, which, in turn, are jointly diagonalizable. 
We find closest commuting operators (CCOs) to a given pair of Laplacians, and use their eigendecomposition for multimodal spectral geometric constructions. 
We showed the application of our approach to several problems in pattern recognition and shape analysis.

We see several avenues to extend the work presented in our paper. 
First, our approach raised an open theoretical question, whether Huaxin Lin's theorem \cite{Huang_Lin} can be restricted to classes of special matrices, such as Laplacians. 

Second, we considered only unnormalized graph Laplacians. Our approach can be extended to other graph Laplacian, as well as discretizations of the Laplace-Beltrami operator on manifolds, such as the popular cotangent formula \cite{Pinkall93} for triangular meshes. 
More broadly, we can consider other Laplace-like operators \cite{hildebrandt2012modal}, heat, wave \cite{aubry2011wave} or general diffusion operators \cite{Coifman}. 

Third, while we used the $L_2$-norms in optimization problem~(\ref{eq:cost}), one can think of situations where the use of the 
sparsity-inducing $L_1$-norm can be advantageous. One such situation is dealing with point-wise topological noise, where one has to modify a few graph weights to perform `surgery' on the edges.

Fourth,  we considered only undirected graphs with symmetric Laplacians. 
An important task is to extend our method to directed graphs or combinations of directed and undirected graphs. 
From the theoretical standpoint, the latter should be possible, as indicated by the following result that builds on the work of Pearcy and Shields \cite{Pearcy1979332}  regarding the commutator of two matrices where one is self-adjoint. We can thus find CCOs, one of which is symmetric and one is not. 

\begin{theo}
If $\Aa$ and $\Bb$ are $n\times n$ real matrices and $\Aa$ is symmetric,
then there are commuting real matrices $\tilde{\Aa}$ and $\tilde{\Bb}$
with $\tilde{\Aa}$ symmetric so that 
\[
\| \Aa - \tilde{\Aa} \| _{\mathrm{F}} + \| \Bb-\tilde{\Bb}^{\prime}\|_{\mathrm{F}}\leq n\sqrt{2} 
\| \Aa\Bb - \Bb\Aa \|_{\mathrm{F}}^{\frac{1}{2}}.
\]
\end{theo}

\begin{proof}
The reader may check that the construction by Pearcy and Shields \cite{Pearcy1979332} produces real matrices that commute when applied to real almost commuting matrices.

Now suppose $\Aa^\Tr=\Aa$. By the real version of Theorem 1 of \cite{Pearcy1979332} there exist $\tilde{\Aa}$ and $\tilde{\Bb}$ with $\tilde{\Aa}$ symmetric and 
\[
\max\left\{ \| \Aa-\tilde{\Aa}\| _{2}, \| \Bb-\tilde{\Bb}\|_{2}\right\}
\leq\frac{\sqrt{n-1}}{\sqrt{2}}\| \Aa\Bb - \Bb\Aa \|_{2}^{\frac{1}{2}}.
\]
Therefore, 
\begin{align*}
\| \Aa-\tilde{\Aa}\|_{\mathrm{F}}+\| \Bb-\tilde{\Bb}\|_{\mathrm{F}} & \leq\sqrt{n}\| \Aa-\tilde{\Aa}\|_{2}+\sqrt{n}\| \Bb-\tilde{\Bb}\|_{2}\\
 & \leq2\sqrt{n}\max\left\{ \| \Aa-\tilde{\Aa}\|_{2},\| \Bb-\tilde{\Bb}\|_{2}\right\}\\
 & \leq\sqrt{2}\sqrt{n^{2}-n}\, \| \Aa\Bb - \Bb\Aa \|_{2}^{\frac{1}{2}}\\
 & \leq\sqrt{2}n\, \| \Aa\Bb - \Bb\Aa \|_{\mathrm{F}}^{\frac{1}{2}}.
\end{align*}
\end{proof}

\section{Acknowledgement}

We are grateful to 
Davide Eynard for assistance with the clustering experiments.  
This research was supported by the ERC Starting Grant No. 307047 (COMET).



\bibliographystyle{plain}\small
\bibliography{laplacians.bib}

\end{document}